\documentclass[letterpaper, 10 pt, conference]{ieeeconf}  
\IEEEoverridecommandlockouts                              

\usepackage{xcolor}
\usepackage{amssymb}
\usepackage{graphicx}
\usepackage{mathtools}
\usepackage{amsmath}
\usepackage{gensymb}
\usepackage{float}
\usepackage{multirow}
\usepackage{makecell}
\usepackage{mathtools}
\usepackage{hyperref}
\hypersetup{
    colorlinks=true,
    linkcolor=blue,
    filecolor=magenta,      
    urlcolor=blue,
    pdftitle={Overleaf Example},
    pdfpagemode=FullScreen,
    }

\usepackage{tabularx}
    \newcolumntype{L}{>{\raggedright\arraybackslash}X}
\usepackage[utf8]{inputenc}
\usepackage[english]{babel}

\usepackage{amsthm}

\usepackage{subcaption}

\newcommand{\no}{\noindent}

\newtheorem{prop}{Proposition}[section]
\title{\LARGE \bf
TERP: Reliable Planning in Uneven Outdoor Environments using Deep Reinforcement Learning
}

\author{Kasun Weerakoon$^{1}$, Adarsh Jagan Sathyamoorthy$^{1}$, Utsav Patel$^{2}$, and Dinesh Manocha$^{2}$
\thanks{This work was supported in part by ARO Grants W911NF1910069, W911NF2110026  and U.S. Army Grant No. W911NF2120076. We acknowledge the support of the Maryland Robotics Center.}
\thanks{$^{1}$ Authors are with Dept. of Electrical and Computer Engineering, University of Maryland, College Park, MD, USA. {\tt\footnotesize kasunw@umd.edu, asathyam@umd.edu}}
\thanks{$^{2}$ Authors are with Dept. of Computer Science, University of Maryland, College Park, MD, USA. {\tt\footnotesize upatel22@umd.edu, dm@cs.umd.edu}}}

\begin{document}

\maketitle
\thispagestyle{empty}
\pagestyle{empty}

\begin{abstract}
We present a novel method for reliable robot navigation in uneven outdoor terrains. Our approach employs a fully-trained Deep Reinforcement Learning (DRL) network that uses elevation maps of the environment, robot pose, and goal as inputs to compute an attention mask of the environment. The attention mask is used to identify reduced stability regions in the elevation map and is computed using channel and spatial attention modules and a novel reward function. We continuously compute and update a navigation cost-map that encodes the elevation information or the level-of-flatness of the terrain using the attention mask. We then generate locally least-cost waypoints on the cost-map and compute the final dynamically feasible trajectory using another DRL-based method. Our approach guarantees safe, locally least-cost paths and dynamically feasible robot velocities in uneven terrains. We observe an increase of 35.18\% in terms of success rate and, a decrease of 26.14\% in the cumulative elevation gradient of the robot's trajectory compared to prior navigation methods in high-elevation regions. We evaluate our method on a Husky robot in real-world uneven terrains ($\sim 4m$ of elevation gain) and demonstrate its benefits. 


\end{abstract}

\section{Introduction} \label{sec:intro}
Autonomous mobile robots have increasingly been used for many real-world field applications such as indoor and outdoor surveillance, search and rescue, planetary/space exploration, large agricultural surveys, etc. Each of these applications requires the robot to operate in different kinds of terrains, which can be characterized by visual features such as color and texture, and geometric features such as elevation changes, slope, etc.

A robot's stability, which includes its pitch and roll angles being within certain limits, is predominantly dictated by the elevation changes/unevenness and slope of the terrain \cite{Fankhauser2018ProbabilisticTerrainMapping}. For reliable navigation, robots need to identify unsafe elevation changes and plan trajectories along planar regions to a large extent. However, sensing and navigating in uneven unstructured environments can be challenging because we do not have a complete model of the terrain with all the elevation information \cite{Silver-2010-10544,siva2021robot}. Rather this information is gathered using camera or LiDAR sensors as the robot navigates. In addition, elevation changes cannot be adequately inferred only from visual features of the environment \cite{josef2020deep}. Prior works have addressed this by using grid-based data structures such as Octomaps \cite{octomap} and elevation maps \cite{Fankhauser2014RobotCentricElevationMapping}, which are 2D grids that contain the maximum elevation (in meters) at each grid.    

Many other techniques have been proposed in the outdoor domain including semantic segmentation-based perception methods \cite{ganav,ss1-gonet,ss2-offseg,ss3-offroadtranseg}, and Deep Reinforcement Learning-based navigation methods \cite{oliveira2021three,guastella2021learning}. Semantic segmentation methods in this domain typically classify terrains in images based on whether they are traversable for a robot \cite{ganav,ss1-gonet,ss2-offseg} by learning the visual features corresponding to the color and texture of different objects. However, they are trained using human annotations which may not be suitable for all types of robots with different dynamic constraints. Therefore, a terrain that is traversable for a robot could be misclassified as non-traversable. As a result, navigation methods that use these segmentation methods for perception could be overly conservative for robot navigation and may result in sub-optimal trajectories \cite{badgr}.

\begin{figure}[t]
      \centering
      \includegraphics[width=7.15cm,height=6cm]{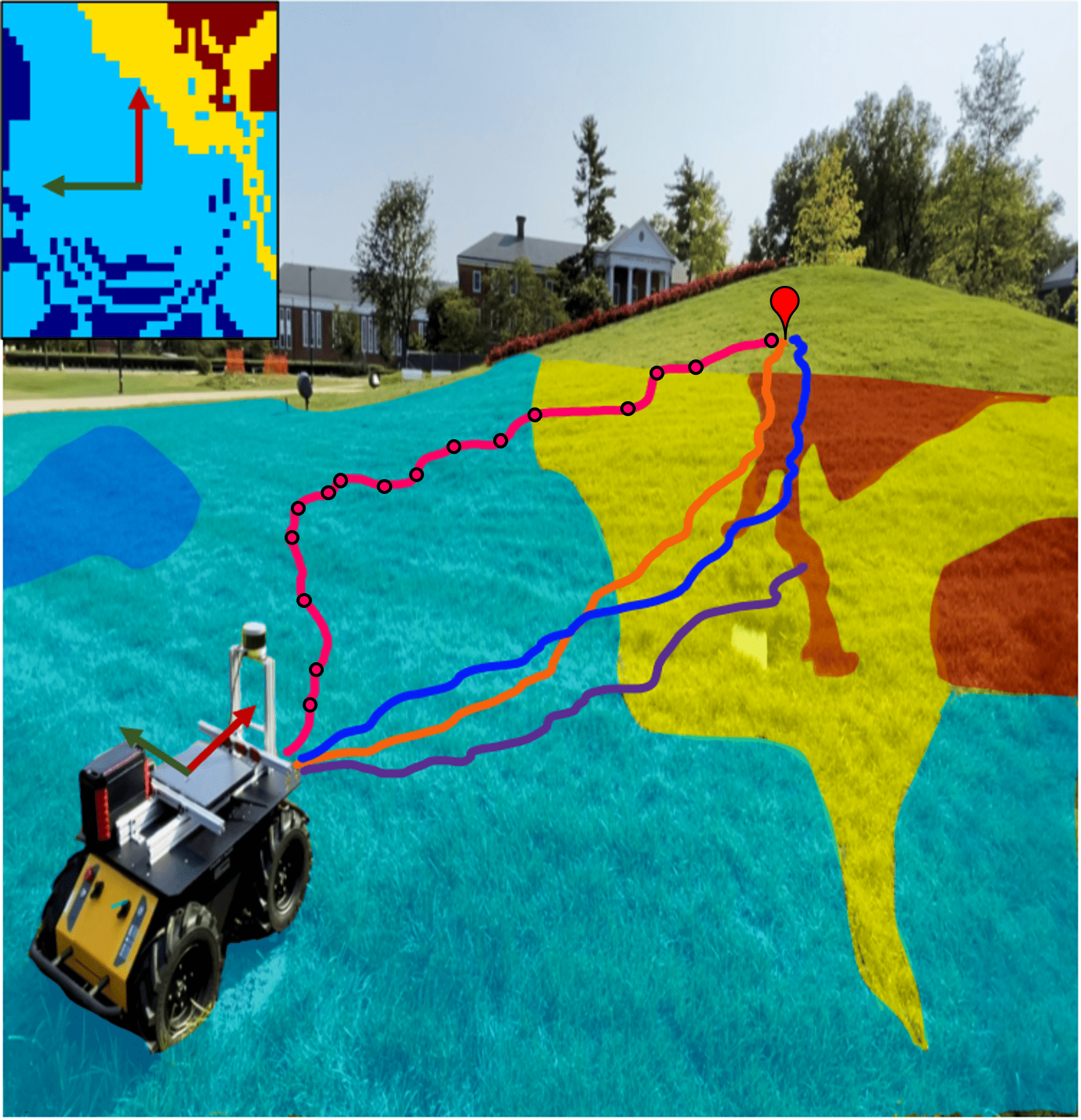}
      \caption {\small{Robot trajectories while navigating in an uneven terrain (elevation gain $\ge 3m$) using four different methods: our method TERP (pink), Ego-graph (orange), DWA (Blue) and velocities from the end-to-end Attn-DRL method (violet). TERP computes an attention mask to highlight reduced-stability regions for navigation, which are then used to compute and update navigation cost-maps continuously (e.g., cost-map computed at the robot's start location is shown in the top left) , with the unsafe regions where the robot’s pitch and roll angle could exceed the stability limit are highlighted in yellow and red (superimposed on the terrain). TERP then generates waypoints (pink points) that are dynamically feasible and reachable by locally least-cost paths. Other methods navigate through unsafe regions with high elevation gradients which could lead to unstable robot orientations. TERP leads to trajectories with low elevation gradients (26.14\% lower than prior navigation methods), leading to higher rate of reaching the goal (35.18\% increase).   
      }}
      \label{fig:cover-image}
      \vspace{-15pt}
\end{figure}

DRL-based navigation approaches \cite{Hu_sim2_real,ganganath2015constraint}, have been used for robot navigation in crowded scenes with dynamic obstacles. However, they are mostly limited to navigation on flat surfaces and do not account for varying elevation in outdoor terrains. Moreover, these DRL methods are typically trained using simulations and may face performance degradation in real-world environments \cite{Hu_sim2_real}. Such methods also cannot provide any optimality guarantees on the robot's trajectories. 

\textbf{Main contributions:} We present TERP (Terrain Elevation-based Reliable Path Planning), a novel method for navigating a robot in a reliable and stable manner in uneven outdoor terrains. That is, our approach identifies unsafe regions (with high elevation gradients) and computes least-cost waypoints to avoid those regions. Our overall approach uses a continuous stream of local elevation maps, robot pose, and goal-related vectors as inputs. The novel components of our approach include:


\begin{itemize}

\item A DRL network (Attn-DRL) that uses a Convolutional Block Attention Module (CBAM) \cite{CBAM} and a reward function to learn appropriate terrain features in the local elevation map that indicate \textit{reduced-stability regions} in the environment. The Attn-DRL network is primarily used for perception. After training, for any given input frame, we extract an attention mask from an intermediate feature vector in the Attn-DRL network. Our attention mask is not biased by human annotations and leads to an increase of 31.9\% in the success rate for reaching the goal in high-elevation environments compared to not using attention, and a decrease of up to 75\% in cumulative elevation gradient of the robot's trajectory.
    
\item A 2D navigation cost-map computed using the extracted attention mask and a normalized elevation map of the environment. Our cost-map encodes a certain region's degree of planarity i.e., a measure of whether the robot's pitch and roll angles will be within safe limits, without the risk of flipping-over. Using the attention-based navigation cost-map leads to a 35.18\% higher rate of reaching the goal in high elevation terrains when compared to prior methods.

\item A method to compute waypoints on the cost-map that are dynamically feasible (i.e., satisfying the kinodynamic constraints of the robot) and reachable by locally least-cost trajectories. These trajectories are used as inputs to DWA-RL \cite{dwa-rl}, a navigation scheme that computes dynamically feasible collision-free robot velocities. Our formulation results in a decrease in the Cumulative Elevation Gradient (CEG) of 26.14\%. Decoupling perception to Attn-DRL network, and navigation to DWA-RL leads to improved overall performance, as compared to using one end-to-end network approach. (see Fig. \ref{fig:cover-image} and Fig. \ref{fig:Testing-scenarios}).  
\end{itemize}

Our overall formulation has significantly lower execution times (takes one-fourth the time) than a state-of-the-art segmentation method \cite{ganav} and, guarantees that the robot's trajectories are locally least-cost and dynamically feasible. We point the reader to \cite{terp_arxiv} for a detailed version of our work.
\section{Related Work}
In this section, we discuss existing work on identifying terrain features and robot navigation on uneven terrains.

\subsection{Identifying Terrain Features}
Early works that addressed terrain identification utilized classification and modeling techniques \cite{kim2006traversability,liu2005multi,govindaraju2005quick}. These include Gaussian mixture models~\cite{nardi2019actively}, Markov random fields ~\cite{wellington2006generative}, terrain mapping~\cite{bolles}, and 3D environment reconstruction~\cite{Keqiang} with perception data from laser range finders \cite{wurm2009improving} or from LiDARs, stereo cameras, and infrared cameras. 

Recent methods have also constructed 3D maps, identified dynamic obstacles for unstructured scene navigation \cite{dalin} and in walkways \cite{takashi}.
Due to the advent of deep neural networks (DNN), classification techniques have been used for terrain understanding \cite{bai2019three,Zurn,wang2021visual}, with some methods using robot-ground interaction data to train DNNs \cite{vulpi2021recurrent}. Most of the visual features extracted by DNNs and external geometric features can vary significantly from one environment to another. Hence, elevation estimation methods were introduced to generalize the terrain understanding problem.


Recent studies such as \cite{guastella2021learning} highlight the importance of terrain elevation data to perform robust outdoor navigation tasks. \cite{ss6-ttm} presents a method to perform terrain semantic segmentation and mapping using point clouds and RGB images for an autonomous excavator. In our work, we utilize an elevation map of the environment computed using point clouds from a 3D LiDAR. 

\subsection{Robot Navigation on Uneven Terrains}
Early works in uneven terrain navigation addressed it using binary classification (obstacle vs. free space) \cite{Laubach}, a continuous obstacle space \cite{pivtoraiko2009differentially}, or potential fields \cite{Shimoda_potential} for high-speed rough terrain navigation. 

Recent developments in deep learning have led to motion planners with high-dimensional input processing capabilities \cite{oliveira2021three}. For instance, \cite{silver2010learning,zakharov2020energy} have incorporated on-board sensory inputs and prior knowledge of the environment to estimate navigation cost functions or cost-maps. In \cite{valencia2013rough}, a nonlinear geometric cost function has been trained to imitate human expert's behavior in uneven terrains. Similarly, an energy-cost model was proposed in \cite{ganganath2015constraint} to generate energy-efficient robot paths on rough terrains.

Numerous end-to-end DRL methods have also been proposed for uneven terrain navigation \cite{guastella2021learning,zhang2018robot,Nguyen}. These methods have used perception inputs such as elevation maps, the robot’s orientation, and depth images to train an A3C-based network \cite{zhang2018robot} or have used RGB images and point clouds to train a multi-modal fusion network \cite{Nguyen}. In \cite{josef2020deep}, the effectiveness of zero to local-range sensing was compared using a rainbow DRL-based local planner. All the aforementioned DRL methods were trained and tested only on simulated uneven terrains. Hence, transferring such methods into a real robot while maintaining comparable navigation performance can be challenging \cite{Hu_sim2_real}. In our method, we decouple perception (to Attn-DRL network) and navigation to two separate modules (see Fig.\ref{fig:System_Block_Diagram}). This ensures that our method's real-world performance is comparable to its simulation performance.

\section{Background}

\subsection{Notations, and Definitions}
We highlight the symbols and notation we use in Table \ref{tab:symbol_defn}. For our formulation, we consider a differential drive robot mounted with a 3D LiDAR with a range of $r_{sense}$, and a $360\degree$ field of view. We assume that the robot's start and goal locations are in stable regions i.e., the robot's roll and pitch angles are within a safe limit. 
We use 2D elevation maps ($\mathbf{E^t}$) (from processed point cloud data) obtained from a 3D LiDAR. $E^t(i, j)$ denotes elevation value at the $(i, j)$ grid coordinate (see Fig. \ref{fig:Costmap_computataion}). All 2D data structures in our work ($\mathbf{E^t, E^t_{N}, A^t, C^t}$) are regarded as $n \times n$ matrices or grids with the robot located at the center. We transform the indices $(i, j)$ such that the robot's position corresponds to $(i, j) = (0, 0)$. We obtain real-world positions $(x, y)$ from indices $(i, j)$ as 
\vspace{-5pt}
\begin{equation}
    \vspace{-10pt}
    (x, y) = res * (i, j).
    \label{eqn:grid-world-conversion}
\end{equation}

\begin{table}[]
\centering
\begin{tabularx}{\linewidth}{|c|L|} 
\hline
\textbf{Symbols} & \textbf{Definitions}  \\
\hline
$\mathbf{E^t}$ & An $n \times n$ 2D local elevation map of the sensed environment at time instant t \\
\hline
\textit{c} & Ground clearance of the robot \\ 
\hline
$\mathbf{E^t_N}$ & Elevation map normalized based on ground clearance at time instant t\\ 
\hline
$\mathbf{A^t}$ & Attention mask obtained from an intermediate feature vector at time instant t\\
\hline
$\mathbf{C^t}$ & Navigation cost-map at time instant t\\
\hline
res & Resolution to convert index locations to real-world locations \\
\hline
$r_{sense}$ & Radius (in meters) of the robot's sensing region \\ 
\hline
$e_{max}$, $e_{min}$ & Max and min elevation values in $E^t(x, y)$ (in meters) \\  
\hline
$d_{goal}$ & Distance between the robot and its goal \\
\hline
$\alpha_{goal}$ & Angle between the robot's heading direction and the direction to the goal \\
\hline
$\alpha_{relative}$ & Angle between the robot's current location and the goal w.r.t. the robot's start location \\
\hline
$\textbf{g}$ & Goal location relative to the robot\\ 
\hline
$\phi$, $\theta$ & Roll and pitch angles of the robot \\ 
\hline

\end{tabularx}
\caption{\small{List of symbols used in our approach.}}
\label{tab:symbol_defn}
\vspace{-15pt}
\end{table}

\subsection{Deep Reinforcement Learning} \label{sec:DRL-background}
Our Attn-DRL network is based on the Deep Deterministic Policy Gradient (DDPG) algorithm \cite{ddpg} combined with Convolutional Neural Network layers and CBAM \cite{CBAM} (see Section \ref{sec:cbam}). We choose DDPG because it uses an actor-critic network, which leads to stable convergence to the optimal policy compared to policy-based or value-based methods. We briefly describe the input and action spaces of our DRL network here and our network architecture and reward function in Section \ref{sec:our-method}. The perception input consists of a 2D normalized elevation map $\mathbf{E^t_{N}}$ (see Equation \ref{eqn:normalized-E}). The robot pose and goal related inputs include $d_{goal}$, $\alpha_{goal}$, $\alpha_{relative}$, $|\phi|$, $|\theta|$, and the elevation gradient vector $\mathbf{h}$ in the robot's heading direction, defined as follows: 

\vspace{-5pt}
\begin{equation}
    \mathbf{G} = ||\nabla \mathbf{E^t}||_2 , \quad \mathbf{h} = G(n/2:1, 0).
\label{eqn:gradient-vector}
\end{equation}

\no where $\nabla$ denotes the gradient operation resulting in an $n \times n$ $\mathbf{G}$ matrix and $G(n/2:1, 0)$ denotes all the values in rows $n/2$ to 1 in the $0^{th}$ column (according to our indexing conventions). Our Attn-DRL network is end-to-end with its action space being a linear and angular robot velocity pair. 

Although DRL methods are useful for training certain robot behaviors in simulations, their performance degrades when transferred to real-world scenarios. In addition, the velocities computed by the end-to-end DRL network cannot guarantee dynamic feasibility or optimality (in terms of any metric). Therefore, we choose not to use the output velocities from Attn-DRL, and instead decouple perception and navigation. We use Attn-DRL purely for perception by extracting an attention mask from an intermediate feature vector $F_{ref}$ in the network (see Section \ref{sec:network-arch} and Fig. \ref{fig:Network_architecture}), and then computing a navigation cost-map. Computing waypoints and robot velocities for navigation is performed in a separate module (see Fig. \ref{fig:System_Block_Diagram}).

\subsection{Convolutional Block Attention Module} \label{sec:cbam}

To highlight \textit{reduced-stability regions} in the environment's elevation map, we use CBAM, a light-weight attention module that can be integrated with any CNN architecture \cite{CBAM}. Given a feature vector (say an output of a CNN) with a certain number of channels, CBAM sequentially applies attention modules along the channels and then along the spatial axis (different parts on the feature vector) to obtain a refined feature vector (see Fig. \ref{fig:Network_architecture}). This leads to our Attn-DRL network learning the features in the elevation map that are relevant for reliable navigation.


\section{TERP: Terrain Elevation-based Robot Path Planning} \label{sec:our-method}
We explain the three major stages of our method: 1. training the Attn-DRL network and extracting the attention mask, 2. computing a navigation cost-map and 3. generating locally least-cost trajectories. 

\begin{figure}[t]
      \centering
      \includegraphics[width=8cm,height=3.25cm]{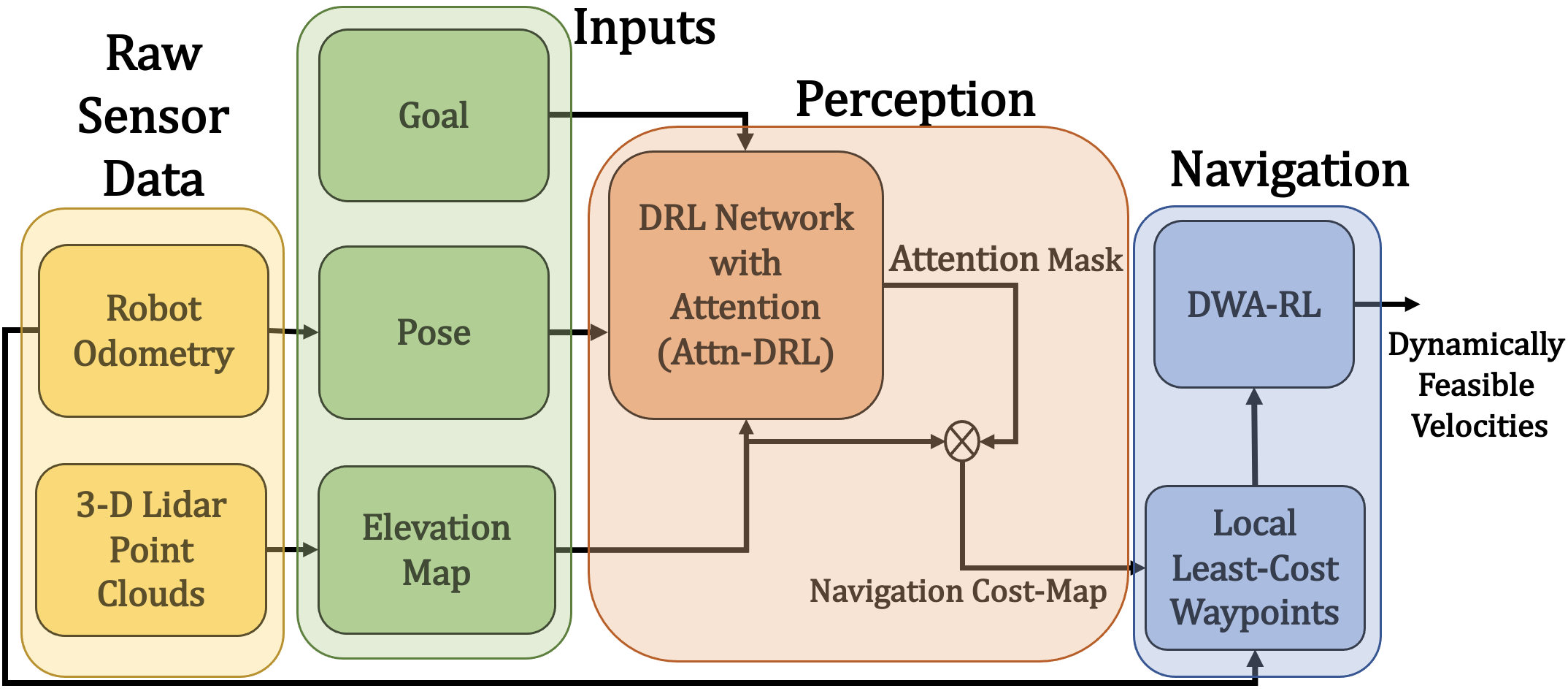}
      \caption {\small{TERP's overall system architecture. We extract an attention mask from the Attn-DRL network continuously for each frame of the input and compute the navigation cost-map. Our architecture decouples perception and navigation to different modules instead of using an end-to-end DRL network, resulting in improved performance. It also ensures that our method performs comparably in both simulated and real-world environments. The advantages of this formulation are highlighted in Section \ref{sec:analysis}.}}
      \label{fig:System_Block_Diagram}
      \vspace{-10pt}
\end{figure}

\subsection{Normalized Elevation Map}
We use existing software packages to process raw point cloud data from a 3D LiDAR and obtain an $n \times n$ local elevation map $\mathbf{E^t}$ around the robot. We employ nearest neighbor interpolation to interpolate values for missing points based on nearby points.

A robot with higher ground clearance \textit{c} can navigate through higher degrees of unevenness than a robot with lower ground clearance. We account for this difference by computing a normalized local elevation map as, $\mathbf{E^t_{N}} = $

\vspace{-10pt}
\begin{equation}
\begin{cases}
(E^t(i, j) - c) + 0.1(e_{max} - e_{min}) \,\, &\text{if} \,\, E^t(i, j) - c > 0, \\
(E^t(i, j) - c) \qquad &\text{Otherwise.}
\end{cases}
\label{eqn:normalized-E}
\vspace{-5pt}
\end{equation}

\subsection{Computing Attention Masks}
We use $\mathbf{E^t_N}$ and the other input vectors (Section \ref{sec:DRL-background}) to train the Attn-DRL network. We explain its novel network architecture and the reward function used for training in simulated environments.

\subsubsection{Network Architecture} \label{sec:network-arch}
Our network architecture (Fig. \ref{fig:Network_architecture}) has two branches. On the first branch, we use a 2D convolutional (conv-) layer of dimensions $40 \times 40 \times 8$ to process the input normalized elevation map (with $n = 40$). This outputs a feature vector with eight channels, which are passed on to the channel and spatial attention modules sequentially. Then, the channels are passed through another conv-layer. The channel and spatial attention modules weigh the features in the different channels and locations in each channel to identify \textit{reduced-stability regions} relevant for reliable navigation. The resultant output is a refined feature vector ($F_{ref}$) with dimensions $40 \times 40 \times 8$, which is finally passed through a dimension reduction conv-layer.

On the second branch, we concatenate the other inputs ($d_{goal}$, $\alpha_{goal}$, $\alpha_{relative}$, $|\phi|$, $|\theta|$, and $\mathbf{h}$ vector) to obtain a 1D vector of size $n/2 + 5$ (25 in this case). This is processed using a fully connected layers of dimensions shown in Fig. \ref{fig:Network_architecture}. The outputs from the two branches are concatenated through several fully connected layers to finally obtain the robot's linear and angular velocities (which are used only for training and comparison). We use $ReLU$ activation in the hidden layers and $tanh$ activation at the output layer. 

\begin{figure}[t]
      \centering
      \includegraphics[width=\columnwidth,height=3.3cm]{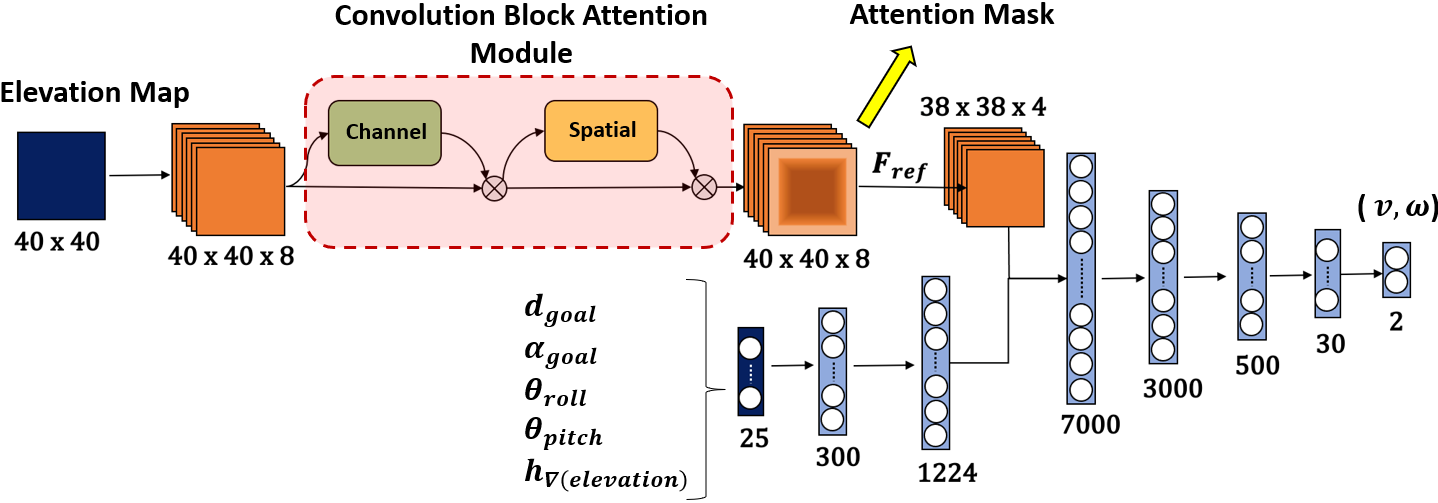}
      \caption {\small{The network architecture of our Attn-DRL network. The dark blue layers represent our normalized elevation map and the vector containing the robot's pose, goal, and elevation gradient-related vectors. The orange layers represent convolutional layers, and the light blue layers are fully-connected layers. The green and yellow blocks are the channel and spatial attention modules, respectively. The dimensions of each layer are mentioned next to it. Our network run in real-time with an execution time of $\sim 0.08$ seconds.}}
      \label{fig:Network_architecture}
      \vspace{-15pt}
\end{figure}

\subsubsection{Reward Functions}
The reward function is used to shape the velocities of the end-to-end attn-DRL network during training, i.e., to reward desirable and penalize undesirable robot actions/velocities. This in turn trains the intermediate feature vectors (such as $F_{ref}$) in the network to learn features that are relevant for reliable navigation (leading to high rewards). The total reward collected by the robot for performing an action at any instant is given as,
\begin{equation}
    R_{tot} = \beta_1 R_{dist} + \beta_2 R_{head} + \beta_3 R_{stable} + \beta_4 R_{grad}.
    \vspace{-5pt}
\end{equation}

\no Here, the different $\beta$ are the weighing factors for the different reward components. $R_{dist}$ and $R_{head}$, are the penalties for the robot moving away from its goal. They are defined as,
\begin{equation}
    \begin{split}
        R_{dist} = - d_{goal}, \quad R_{head} = - |\alpha_{goal}|.
    \end{split}
    \vspace{-10pt}
\end{equation}

\no The novel terms of our reward function are $R_{stable}$ and $R_{grad}$ which encourage the network to learn features to maintain the robot's stability and avoid regions with high elevation gradients. $R_{stable}$ is the reward for the robot maintaining stability (having low roll and pitch angles), and $R_{grad}$ is the penalty for heading in a direction with a high elevation gradient. They are defined as,
\begin{equation}
    R_{stable} = \cos^2{\phi} + \cos^2{\theta}, \quad R_{grad} = - \mathbf{w} \cdot \mathbf{h},
    \vspace{-5pt}
\end{equation}


\no where $\mathbf{w}$ is a weight vector with weights in ascending order, and $\mathbf{h}$ is the heading gradient vector defined in Equation \ref{eqn:gradient-vector}. The $\mathbf{w}$ vector weighs elevation changes that are closer to the robot higher than the ones farther away. 

\subsubsection{Attention Mask Extraction}
As mentioned in Section \ref{sec:DRL-background}, we extract an attention mask $\mathbf{A}$ from the refined feature vector $F_{ref}$. Since \textbf{$F_{ref}$}'s different channels are already weighted by the channel and spatial attention modules, the attention mask can be obtained by an unweighted summation along the channels. This is,

\vspace{-10pt}
\begin{equation}
    \mathbf{A}^t(i, j) = \sum_{i=1}^{8} F_{ref}(x, y, i).
\end{equation}

The attention mask weighs elevation changes in the direction the robot is moving or turning higher than all other directions. This is depicted in Fig.\ref{fig:Costmap_computataion}(b), where the attention mask corresponds to the normalized elevation map presented in Fig.\ref{fig:Costmap_computataion}(a). The robot turns to its right to head towards its goal, and the high elevation changes on the right are highlighted in the attention mask.

\subsection{Computing Navigation Cost-map}
Since the normalized elevation map and the attention mask represent the environment's elevation values and the weights for the different regions at a time instant $t$, we utilize them to compute a navigation cost-map as,
\vspace{-5pt}
\begin{equation}
    \mathbf{C^t} = \mathbf{E^t_{N}} \odot \mathbf{A^t}.
    \vspace{-5pt}
\end{equation}

\no Here $\odot$ represents element-wise multiplication. The computed costmap is shown in Fig.\ref{fig:Costmap_computataion}(c). To completely avoid certain high-cost regions while navigating (costs higher than a certain threshold $C_{max}$), we perform the following operation.
\vspace{-5pt}
\begin{equation}
    C^t(i, j) = \infty \qquad if \qquad C^t(i, j) > C_{max}. 
    \vspace{-5pt}
\end{equation}

\no $C_{max}$ is set not only based on the elevation in a region, but also on the robot's maximum torque capabilities. Therefore, all regions not reachable by the robot are assigned infinite cost.

\begin{figure*}[t]
      \centering
      \includegraphics[width=16cm,height=3.75cm]{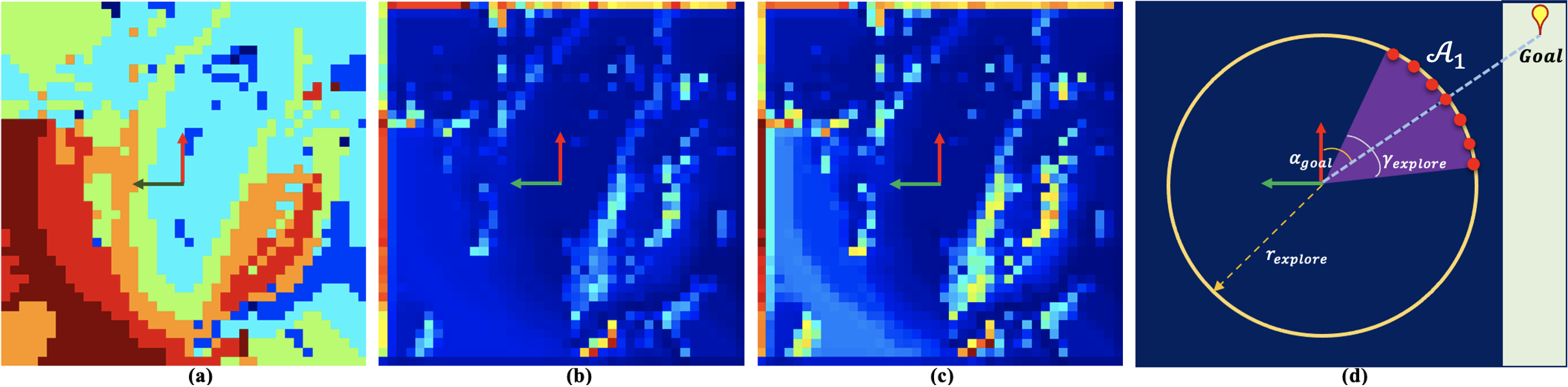}
      \caption {\small{(a) The normalized input elevation map $\mathbf{E^t_N}$. The colors range from light shades of blue (low elevation) to dark shades of red (high elevation) indicating elevation relative to the robot; (b) Attention mask extracted from Attn-DRL network. The light blue regions highlight the \textit{reduced-stability regions} in the direction of the goal (see (d)). The attention mask \textbf{$\mathbf{A^t}$} does not focus on high elevation regions in the left bottom since they are not relevant to the robot's current goal direction (top right); (c) Cost-map computed from $\mathbf{E^t_N}$ and $\mathbf{A^t}$. The lighter regions have high cost.; (d) Least-cost waypoint computation with the other parameters represented. The red dots represent the candidate waypoints on the arc $\mathcal{A}_1$}.}
      \label{fig:Costmap_computataion}
      \vspace{-15pt}
\end{figure*}

\subsection{Computing Least-cost Waypoints}
Consider a local cost-map $\mathbf{C^t}$ and a goal location $\textbf{g}$ that lies outside it. Choosing least-cost waypoints (some index (i, j)) inside $C^t(i, j)$ that lead the robot towards its goal is non-trivial. This is because the robot operates without any global knowledge of the environment. In addition, the chosen waypoint must be safely reachable by the robot. To address this, we formulate the following method.

Given a robot with its goal at $\textbf{g}$, we consider an approximate circle $\mathcal{C}_{explore}$ in $\mathbf{C^t}$ centered at the robot with a radius $r_{explore} < r_{sense}$.  Next, we consider an arc $\mathcal{A}_1^t$ on the circumference of $\mathcal{C}_{explore}$ that makes an angle of $\gamma_{explore}$ at the center with the goal vector $\textbf{g}$ bisecting it (see Fig.\ref{fig:Costmap_computataion}(d)). We use all locations $(i, j) \in \mathcal{A}_1^t$ as candidate waypoints. We then choose the set of \textit{least-cost} waypoints as,
\begin{equation}
    L^t_{min} = \{\underset{(i, j) \in \mathcal{A}_1^t}{\operatorname{argmin}} \left(cost((i, j))\right)\}.
    \vspace{-5pt}
\end{equation}

\no Here, for $cost((i, j))$ is the cumulative cost of navigating to $(i,j)$ from the robot's location $(0, 0)$ in the cost-map by trajectories computed by Dijkstra's algorithm \cite{dijkstra}. If multiple index locations exist in $L^t_{min}$, we choose the waypoint by transforming it to real-world coordinates (see Equation \ref{eqn:grid-world-conversion}) as,
\begin{equation}
    x^*, y^* = \underset{x, y \in (res \cdot L^t_{min})}{\operatorname{argmin}} \left(dist((x, y), \textbf{g})\right).
    \label{eqn:least-cost-waypoints}
    \vspace{-5pt}
\end{equation}

\no In the highly unlikely event that $L^t_{min} = \emptyset$ (implying all trajectories have infinite cost), we expand the arc $\mathcal{A}^t_1$ to $\mathcal{A}^t_2$, which makes an angle of $2\cdot\gamma_{explore}$ at $\mathcal{C}_{explore}$'s center. We then consider $(i, j) \in \mathcal{A}^t_2 \setminus \mathcal{A}^t_1$ as candidate waypoints and repeat the procedure. 

When the terrain is planar, the robot can safely have a larger exploration circle and choose waypoints on its circumference. This significantly reduces the number of waypoints that need to be computed before reaching the goal. Therefore, $r_{explore}$ is formulated as a function of the costs within the circle as, 
\vspace{-5pt}
\begin{equation}
    r_{explore} = k_1 + k_2\cdot\frac{1}{mean(C^t(i, j))}; \quad \forall i, j \in \mathcal{C}_{explore}.
\end{equation}

\no where $k_1$ and $k_2$ are constants. If the costs within $\mathcal{C}_{explore}$ are low (implying a planar surface), the formulation expands the radius.

\begin{prop}
Our method always navigates the robot in a dynamically feasible trajectory with a low cumulative elevation gradient locally.
\end{prop}
\begin{proof}
This result follows from the fact that the waypoints computed using Equation \ref{eqn:least-cost-waypoints} are reachable with the least trajectory cost locally. In addition, since our cost-map has infinite costs for regions that are unsafe and unreachable based on the robot's torque requirements, such regions are guaranteed to be avoided while computing trajectories.
\end{proof}

\subsection{Dynamic Feasibility}
We use DWA-RL \cite{dwa-rl} to compute robot velocities to follow the least-cost trajectory computed by Dijkshtra's algorithm. DWA-RL guarantees that the velocities are collision-free and obey the robot's kinodynamic constraints. Therefore, the robot can also avoid collisions with dynamic obstacles in the environment. Fig. \ref{fig:System_Block_Diagram} shows how the components in our method are connected. Our approach decouples perception to our DRL network with the attention modules and the navigation to our local waypoint computation and DWA-RL. This architecture leads to superior performance and guarantees least-cost waypoints and the dynamic feasibility of the final velocities.  
\section{Results and Analysis}
We explain our method's implementation in simulations and on a real-world robot. We then explain our evaluation metrics and discuss our inferences. 

\subsection{Implementation}
Our Attn-DRL network is implemented in Pytorch. We use simulated uneven terrains with a Clearpath Husky robot created using ROS Melodic and Unity game engine to train the Attn-DRL network. The simulated Husky robot is mounted with a Velodyne VLP16 3D LiDAR. The network is trained in a workstation with an Intel Xeon 3.6 GHz processor and an Nvidia Titan GPU. 

We replicate the simulated robot's setup on a real Husky robot. In addition to the LiDAR, the robot is mounted with a laptop with an Intel i9 CPU and an Nvidia RTX 2080 GPU. We use the Octomap \cite{octomap} and Grid-map \cite{GridMapLibrary} ROS packages to obtain the elevation map $\mathbf{E^t}$ of size $40 \times 40$ using the point clouds from the Velodyne ROS package. Our network is computationally light enough to run real-time on the laptop along with the aforementioned ROS packages.    

\subsection{Evaluation Metrics}
We use the following metrics to compare our method's navigation performance with the Dynamic Window Approach (DWA) \cite{DWA}, the ego-graph \cite{ego_graph} method used in \cite{josef2020deep}, and our end-to-end Attn-DRL network (without decoupling perception and navigation). The ego-graph method uses a cost function based on $\alpha_{goal}$ and the terrain's cumulative elevation gradient along a set of local trajectories to find the minimum cost trajectory. We add a distance-to-goal cost to this cost function and use it for additional comparison (ego-graph+). 

\no \textbf{Success Rate} - The number of times the robot reached its goal while avoiding \textit{reduced-stability regions} and collisions over the total number of attempts.

\no \textbf{Cumulative Elevation Gradient (CEG)} - The summation of the elevation gradients experienced by the robot along a trajectory (i.e., $||\nabla z_{rob}||$, where $\nabla z_{rob}$ represents the robot's elevation at any time instant).

\no \textbf{Normalized Trajectory length} - The robot's trajectory length normalized using the straight-line distance to the goal for both successful and unsuccessful trajectories. 

\no \textbf{Goal Heading Deviation} - The cumulative angle difference between the robot's heading and the goal along a trajectory (i.e., $\sum_{i = 1}^{t_{goal}} \alpha^i_{goal}$, where $t_{goal}$ is the time to reach the goal).   
    

\subsection{Testing Scenarios}
We evaluate our elevation based navigation framework and compare it with prior methods in four scenarios(see Fig.\ref{fig:Testing-scenarios}). 

\begin{itemize}
\item \textbf{Low-Elevation} - Maximum elevation gain $\le 1m$.

\item \textbf{Medium-Elevation} - Elevation gains between $1-2m$.

\item \textbf{High-Elevation} - Maximum elevation gain $\ge 3m$.

\item \textbf{City-curb} - A city environment with curbs higher than the robot's ground clearance that it must avoid.
\end{itemize}

\begin{figure*}[t]
      \centering
      \includegraphics[width=\textwidth,height=7.2cm]{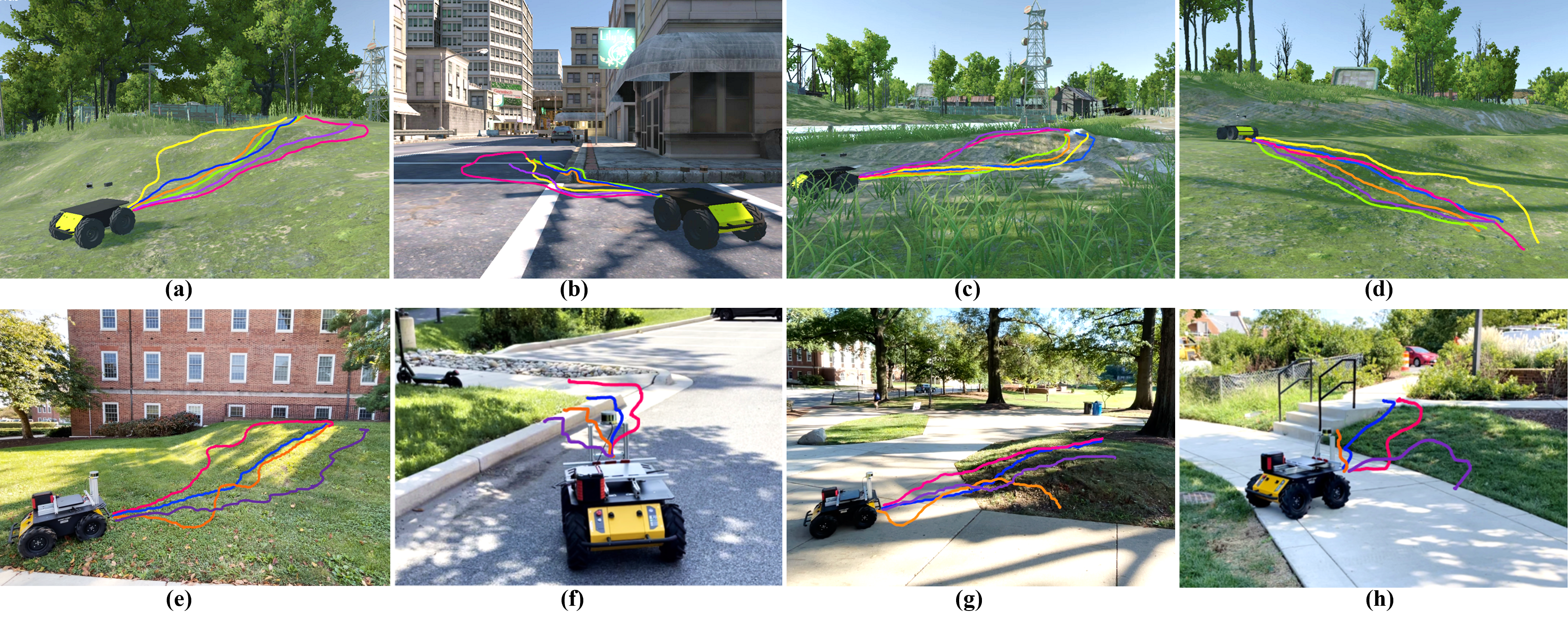}
      \caption {\small{Robot trajectories when navigating in different simulated and real-world uneven terrains using TERP (pink), TERP w/o attention (yellow), end-to-end Attn-DRL network (violet), ego-graph (orange), Ego-graph+ (green) and DWA (blue).  
 (a) High-elevation; (b) City-curb;(c) Low-elevation;(d) Medium-elevation; (e) real-world medium-elevation; (f) real-world curb (g) real-world medium-elevation; (h) real-world medium-elevation with obstacle regions;. For the real-world scenarios, we show four trajectories corresponding to the aforementioned colors. We observe that TERP computes trajectories with low elevation gradients in uneven terrains and is capable of handling challenging curb scenarios. }}
\label{fig:Testing-scenarios}
\vspace{-10pt}
\end{figure*}


\begin{table}
\resizebox{\columnwidth}{!}{%
\begin{tabular}{ |c |c |c |c |c |c |} 
\hline
\textbf{Metrics} & \textbf{Method} & \multicolumn{1}{|p{1cm}|}{\centering \textbf{Low} \\ \textbf{Elevation}} & \multicolumn{1}{|p{1cm}|}{\centering \textbf{Medium} \\ \textbf{Elevation}} & \multicolumn{1}{|p{1cm}|}{\centering \textbf{High} \\ \textbf{Elevation}} & \multicolumn{1}{|p{1cm}|}{\centering \textbf{City} \\ \textbf{Curb}}\\ [0.5ex] 
\hline
\multirow{6}{*}{\rotatebox[origin=c]{0}{\makecell{\textbf{Success}\\\textbf{Rate (\%)}}}} 
 & DWA \cite{DWA} & 71 & 83 & 61 & 39   \\
 & Ego-graph \cite{ego_graph} & 69 & 79 & 62 & 29 \\
 & Ego-graph+ & 67 & 81 & 54 & 31 \\
 & End-to-end Attn-DRL & 70 & 79 & 60 & 45 \\
 & TERP without Attention & 76 & 83 & 67 & 47 \\
 & TERP with Attention & \textbf{82} & \textbf{85} & \textbf{73} & \textbf{62} \\
\hline

\multirow{6}{*}{\rotatebox[origin=c]{0}{\makecell{\textbf{CEG}}}}  
 & DWA \cite{DWA}  & 0.68 & 1.34 & 2.49 & 0.298   \\
 & Ego-graph \cite{ego_graph} & 0.73 & 1.25 & 2.83 & 0.407 \\
 & Ego-graph+ & 0.72 & 1.23 & 2.64 & 0.380 \\
 & End-to-end Attn-DRL & 0.65 & 1.24 & 2.15 & 0.084 \\
 & TERP without Attention & 0.70 & 1.47 & 2.13 & 0.032 \\
 & TERP with Attention & \textbf{0.61} & \textbf{1.22} & \textbf{2.09} & \textbf{0.008} \\
\hline

\multirow{6}{*}{\rotatebox[origin=c]{0}{\makecell{\textbf{Norm.}\\\textbf{Traj.}\\\textbf{Length}}}} 
 & DWA \cite{DWA}  & 1.15 & \textbf{1.05} & 1.17 & 1.22   \\
 & Ego-graph \cite{ego_graph} & 1.06 & 1.19 & 1.10 & 1.01 \\
 & Ego-graph+ & \textbf{1.03} & 1.20 & \textbf{1.01} & 1.04 \\
 & End-to-end Attn-DRL & 1.42 & 1.18 & 0.98 & 1.12 \\
 & TERP without Attention & 1.16 & 1.23 & 1.33 & \textbf{0.91} \\
 & TERP with Attention & 1.39 & 1.21 & 1.24 & 1.55 \\
\hline

\multirow{6}{*}{\rotatebox[origin=c]{0}{\makecell{\textbf{Goal }\\\textbf{Heading Deviation}}}}  & DWA \cite{DWA} & 49.24 & \textbf{10.34} & \textbf{7.19} & 13.18   \\
 & Ego-graph \cite{ego_graph} & \textbf{29.56} & 43.02 & 8.43 & \textbf{12.14} \\
 & Ego-graph+ & 33.26 & 69.23 & 43.19 & 24.16 \\
 & End-to-end Attn-DRL & 83.51 & 65.42 & 71.86 & 46.58 \\
 & TERP without Attention & 45.93 & 67.51 & 75.78 & 23.08 \\
 & TERP with Attention & 74.46 & 53.98 & 89.11 & 161.91 \\
\hline

\end{tabular}
}
\caption{\small{Relative performance of our method versus other methods on various metrics. TERP consistently outperforms other methods in terms of success rate, and CEG even in challenging high-elevation and city-curb environments. The computed stable trajectories possess longer trajectory lengths and higher deviations from the goal depending on the environment. Our quantitative results also highlight the advantages of the included attention modules in our Attn-DRL network.}
}
\label{tab:comparison_table}
\vspace{-15pt}
\end{table}

\subsection{Analysis and Comparison} \label{sec:analysis}
The results are shown in Table \ref{tab:comparison_table}. We observe that all the methods perform well under low and medium elevation conditions in terms of success rate. However, TERP displays significant improvement in successfully reaching goals in highly-elevated environments. This is due to TERP's ability to navigate through the locally least-cost waypoints even under high elevation conditions. Our end-to-end Attn-DRL method tries to plan trajectories along low-cost regions even though the minimum cost cannot be guaranteed.  Hence, the success rate is lower and the CEG is higher than our TERP method. DWA sometimes avoids the high-elevations in front of the robot by treating them as obstacles. However, when the robot reaches an elevated terrain, DWA cannot recognize the elevations to avoid during navigation. In contrast, ego-graph-based methods select the optimal navigation velocities based on the front elevation gradients. These methods cannot maneuver the robot along the minimum-gradient trajectory because they have a discretized action space.


The CEG values in Table \ref{tab:comparison_table} imply that our method attempts to reach the goal by following the least-cost waypoints, which eventually result in minimum elevation gradients along the trajectory. However, this could lead to longer trajectories and higher deviations from the goal direction than other methods. The results in Table \ref{tab:comparison_table} validate that TERP can achieve better success rates while maintaining comparable trajectory lengths.

We observe that navigating in a city environment with curbs higher than the robot's ground clearance is a challenging task for all the methods. Comparatively small elevation gains of the curb regions could possibly lead to erroneous trajectories. However, TERP with attention outperforms all the other methods in terms of success rate (113\% improvement w.r.t ego-graph) by attending to the critical elevation gradients near the curb (see Fig. \ref{fig:Testing-scenarios}b and f). 

\textbf{Ablation Study for the Attention Module:} We compare navigation performance of TERP with and without the attention module. TERP without attention utilizes input elevation maps as the cost-map to find least-cost waypoints. Hence, trajectories generated by this approach select waypoints based on the locally minimum elevations. In contrast, TERP with attention follows the least-cost waypoints by attending to spatially critical elevation changes along the robot's trajectory (see Fig. \ref{fig:Costmap_computataion}c). Hence, the attention module improves the robot's spatial awareness to perform safe navigation tasks (See CEG in Table \ref{tab:comparison_table}, yellow and pink paths in Fig. \ref{fig:Testing-scenarios}).

\textbf{Computational Complexity:}  We compare TERP's execution times with GANav \cite{ganav}, a state-of-the-art semantic segmentation method. TERP is computationally light enough to perform in real-time with execution times between 0.08-0.1 seconds (10-12.5Hz), while GANav takes 0.32-0.39 seconds on the same processing hardware.

\textbf{Decoupled vs. end-to-end:} 
Based on the end-to-end Attn-DRL network's trajectories in simulation (Fig. \ref{fig:Testing-scenarios}a to \ref{fig:Testing-scenarios}d) and in the real-world (Fig. \ref{fig:cover-image} and Fig. \ref{fig:Testing-scenarios}e to \ref{fig:Testing-scenarios}h), we observe a significant degradation in performance. In contrast, TERP exhibits similar performance in both cases. Additionally, TERP also guarantees least-cost dynamic feasible waypoints, while the end-to-end network cannot.  


\section{Conclusions, Limitations and Future Work}
We present a novel formulation to reliably navigate a ground robot in uneven outdoor environments. Our hybrid architecture combines intermediate output from a DRL network with attention with input elevation data to obtain a cost-map to perform local navigation tasks.  We generate locally least-cost waypoints on the obtained cost-map and integrate our approach with an existing DRL method that computes dynamically feasible robot velocities. We validate our approach in simulation and on real-world uneven terrains and compare it with the other navigation techniques on various navigation metrics. 

Our framework has a few limitations which we expect to address in the future. The effective sensing point cloud percentage can be low especially near steep ditches due to the low reflectance of the LiDAR's rays. In such cases, additional depth sensors need to be integrated to achieve better perception. In addition, boundaries between different surfaces may not be sensed by the elevation map. Using a fusion of RGB image data and point clouds could help address this problem.

\bibliographystyle{IEEEtran}
\bibliography{References}

\end{document}